\documentclass[11pt]{article}

\usepackage[margin=1in]{geometry}
\usepackage{fontspec}
\usepackage{amsmath,amssymb,amsthm}
\usepackage{booktabs}
\usepackage{graphicx}
\usepackage{float}
\usepackage[table]{xcolor}
\usepackage[nocompress]{cite}
\usepackage{hyperref}
\usepackage{microtype}

\hypersetup{
  colorlinks=true,
  linkcolor=blue!45!black,
  citecolor=blue!45!black,
  urlcolor=blue!45!black
}

\newtheorem{assumption}{Assumption}
\newtheorem{lemma}{Lemma}
\newtheorem{proposition}{Proposition}
\newtheorem{corollary}{Corollary}

\newcommand{\runningtitle}{GradRepair-ODE}
\newcommand{\runningauthors}{Z. Bi and X. L. Chia}

\title{\textbf{GradRepair-ODE: Certified Gradient Repair for Neural ODE Training}}
\author{
Ziqian Bi\thanks{\href{https://scholar.google.com/citations?view_op=view_org&hl=zh-CN&org=9642744976452465328}{Purdue University}, West Lafayette, IN, USA (\href{mailto:bi32@purdue.edu}{bi32@purdue.edu}).}
\and
Xin Liang Chia\thanks{\href{https://scholar.google.com/citations?view_op=view_org&hl=zh-CN&org=11940724553757128128}{Rice University}, Houston, TX, USA (\href{mailto:xc71@rice.edu}{xc71@rice.edu}).}
}
\date{}

\begin{document}
\maketitle

\begin{center}
\small
\textbf{Short title:} \runningtitle
\qquad
\textbf{Authors:} \runningauthors
\end{center}

\begin{abstract}
Neural ordinary differential equations use numerical solvers inside the training loop. The solver determines the forward trajectory and also affects the gradient passed to the optimizer. That coupling creates a reliability problem for scientific machine learning and continuous-time generative modeling, including diffusion probability-flow ordinary differential equations and flow-matching models. Under loose step sizes, stiff dynamics, chaotic sensitivity, or event discontinuities, a differentiable ODE pipeline can return a finite gradient whose direction is numerically suspect. We introduce GradRepair-ODE, a reliability framework for checking, repairing, and rejecting ODE gradients at the optimizer step. The method computes several gradient candidates, compares them with directional finite-difference checks and solver diagnostics, diagnoses likely numerical failure modes, repairs selected gradients through path switching or stricter recomputation, and rejects steps whose descent direction cannot be certified. In six synthetic ODE systems, GradRepair-ODE leaves low-risk systems unchanged, repairs Robertson and Lorenz gradients to cosine similarity $1.000$ against a strict reference, reduces unsafe accepted steps from $37$ to $0$, and rejects an event-discontinuous case instead of applying an uncertified update. The paper argues for a simple change in the training contract: an ODE gradient should reach the optimizer with numerical evidence attached.
\end{abstract}



\medskip
\noindent\textbf{SISC section.}
Machine Learning Methods for Scientific Computing.

\section{Introduction}

Neural ordinary differential equations (Neural ODEs) replace a discrete stack of transformations with a continuous-time dynamical system whose trajectory is computed by a numerical ordinary differential equation (ODE) solver \cite{chen2018neuralode}. This formulation gives the model an appealing mathematical interface and allows solver tolerance to trade numerical precision for speed. It also moves a large part of training into numerical analysis: the optimizer does not receive an exact gradient, but a gradient shaped by solver error, interpolation, stiffness, horizon length, and the chosen sensitivity path.

Modern differentiable ODE methods expose both direct backpropagation through solver operations and adjoint-style methods with different memory and accuracy tradeoffs \cite{chen2018neuralode,kidger2021seminorm}. The same numerical issue reaches beyond classical Neural ODE benchmarks. Continuous normalizing flows (CNFs) use ODE dynamics for density modeling \cite{grathwohl2019ffjord}; diffusion probabilistic models and score-based stochastic differential equation (SDE) models connect generative modeling to continuous-time dynamics and probability-flow ODEs \cite{song2021scorebased,ho2020ddpm}; flow matching trains CNFs by regressing vector fields along probability paths and explicitly includes diffusion paths as an important case \cite{lipman2023flowmatching}; rectified flow learns ODE transports that aim to make generative paths straight and solver-efficient \cite{liu2023rectifiedflow}. In all of these settings, a learned vector field is ultimately paired with numerical integration. When such models are trained through differentiable solves or evaluated by ODE-based likelihood or sampling routines, gradient reliability is no longer a niche implementation detail; it is part of the numerical foundation of continuous-time generative modeling.

Prior work has shown that optimize-then-discretize adjoints and discretize-then-optimize gradients can disagree in ways that matter for Neural ODE training \cite{gholaminejad2019anode,onken2020discopt}. Classical numerical analysis has long treated step-size control, dense output, stiffness, backward differentiation formulae, finite-precision effects, and sensitivity analysis as solver-level concerns \cite{dormand1980embedded,hairer1993nonstiff,shampine1997matlab,ascher1998computer,higham2002accuracy,hindmarsh2005sundials,hairer1996stiff,brenan1996dae}. General scientific computing libraries likewise distinguish non-stiff explicit solvers from stiff-aware implicit methods and warn that divergence or unusually many iterations can indicate stiffness \cite{scipy_solve_ivp}. These facts are well understood at the solver level, but the optimizer interface usually receives only the resulting tensor. Once automatic differentiation returns a finite array with the expected shape, most training loops have little additional evidence about its directional reliability.

The failure is easy to miss. A gradient tensor can be finite, shaped correctly, and still point in the wrong direction. If that direction disagrees with other sensitivity paths or with local directional checks, the optimizer may increase the loss while the training loop records only a normal update. The relevant question is whether the current numerical evidence supports the next step.

GradRepair-ODE answers that question at the optimizer boundary. It computes multiple gradient candidates, builds a reliability certificate from path disagreement and finite-difference residuals, diagnoses numerical risk, repairs selected gradients, and certifies whether the final step is accepted. The method does not propose another ODE architecture or another general-purpose solver. It adds a missing interface: gradient trust, repair, and rejection become recorded training events.

The point is sharper than a debugging aid. Solver tolerance, an adjoint method, and a loss curve do not say whether the next update is defensible. GradRepair-ODE places a certificate where numerical approximation becomes optimization action. The certificate can trust a benign gradient, route a repairable gradient to a safer path, or withhold an update whose smooth-gradient evidence has collapsed. This matters in scientific computing because the damaging failure is often silent: the computation continues with a plausible number whose direction has lost numerical meaning.

The experiments use controlled mechanisms rather than decorative applications. The suite contains smooth, stiff, chaotic, event-driven, and neural-vector-field dynamics, so the same reliability layer faces the pathologies it claims to expose. The main result is the conversion of hidden optimizer risk into explicit decisions. GradRepair-ODE leaves low-risk systems untouched, repairs Robertson and Lorenz gradients to cosine similarity $1.000$ against a strict reference, reduces unsafe accepted steps from $37$ to $0$, and rejects the event-discontinuous case when repair would overstate the evidence. A useful reliability method should have exactly this profile: keep benign steps cheap, pay for accuracy when the direction can be recovered, and refuse an update when the differentiated object no longer matches the optimized trajectory.

The paper contributes four pieces. It defines an optimizer-facing gradient reliability certificate for differentiable ODE training, using sensitivity-path disagreement, directional finite differences, and solver diagnostics. It formulates repair as constrained selection among gradient paths, not as a blanket move to the strictest computation. It derives a local descent certificate that connects an empirical gradient-error radius to the next optimizer step. It also reports a controlled numerical study, with gradient clipping and Armijo-style line search included as safeguards, showing when the certificate trusts, repairs, or rejects. The outcome is not a training-stability anecdote. It is a per-step numerical account of the update stream.

\section{Method}

\subsection{Problem Setup}

We consider an ODE parameterized by $\theta$,
\begin{equation}
  \frac{dx}{dt} = f(t, x, \theta), \qquad x(t_0)=x_0,
\end{equation}
with a trajectory-level loss
\begin{equation}
  L(\theta)=\ell(x(t_0:t_1;\theta), y).
\end{equation}
The training loop receives an approximate gradient $\hat g$ produced by a numerical solver and a sensitivity path. GradRepair-ODE tests whether $\hat g$ can support the update $-\eta \hat g$.

The distinction between a derivative and an update is practical, not semantic. A sensitivity path can produce a parameter-space vector even when the trajectory used to construct it is fragile. Let $\mathcal{S}$ denote the solver configuration, including the integration rule, tolerance, step-size policy, interpolation rule, and event handling. Let $\mathcal{P}$ denote the sensitivity path. The returned gradient is better written as
\begin{equation}
  \hat g = G(\theta;\mathcal{S},\mathcal{P}),
\end{equation}
instead of simply $\nabla L(\theta)$. GradRepair-ODE treats the gap between $G(\theta;\mathcal{S},\mathcal{P})$ and $\nabla L(\theta)$ as an observable numerical quantity. The framework is solver-agnostic in the same sense that a line search is optimizer-agnostic: it leaves the underlying numerical method in place and changes the condition under which the optimizer trusts its output.

Each training step produces gradient-path evidence, directional evidence, and solver evidence. Gradient-path evidence compares independent or partially independent approximations of the same derivative. Directional evidence compares inner products against local centered finite differences. Solver evidence records numerical stress, such as step-size collapse, high function-evaluation counts, failed steps, or known discontinuities. A gradient supported by these signals passes cheaply. A gradient that fails them is repaired or withheld.

\subsection{Gradient Candidates}

Each diagnostic step compares the default coarse solver-path gradient with gradients from more accurate recomputation paths. A finer discrete path tests whether the direction survives reduced integration error. A checkpointed-style recomputation path tests whether greater trajectory consistency changes the gradient. A strict path provides a high-cost reference for certification or last-resort repair. This design borrows the memory-computation tradeoff familiar from reverse-mode differentiation, adjoint sensitivity analysis, and checkpointing \cite{griewank2000revolve,hindmarsh2005sundials,serban2005cvodes}. Here the goal is the update decision: accept, repair, or reject before the parameter changes.

The candidate set is ordered by cost and expected numerical fidelity:
\begin{equation}
  \mathcal{C}
  =
  \{g_{\mathrm{coarse}}, g_{\mathrm{disc}}, g_{\mathrm{ckpt}}, g_{\mathrm{strict}}\}.
\end{equation}
The coarse path is the gradient a standard training loop would consume. The discrete path recomputes the solve at smaller steps and checks whether the direction survives refinement. The checkpointed-style path represents a more trajectory-consistent recomputation with additional cost. The strict path is used as a repair target or evaluation reference in the controlled study. In a larger implementation, these symbols may correspond to continuous adjoints, direct differentiation through solver operations, checkpointed discrete adjoints, direct sensitivity equations, or stiff-aware solvers. The claim does not depend on a particular solver; it depends on multiple numerical routes whose disagreement carries information.

\subsection{Reliability Certificate}

For a candidate gradient $g_i$, GradRepair-ODE computes pairwise cosine disagreement
\begin{equation}
  D_{\cos}(g_i,g_j)=1-\frac{\langle g_i,g_j\rangle}{\|g_i\|\|g_j\|+\delta},
\end{equation}
relative norm disagreement
\begin{equation}
  D_{\mathrm{norm}}(g_i,g_j)=
  \frac{\|g_i-g_j\|}{\|g_i\|+\|g_j\|+\delta},
\end{equation}
and directional finite-difference (FD) residual
\begin{equation}
  R_{\mathrm{FD}}(g,v)=
  \frac{|\langle g,v\rangle-\mathrm{FD}(\theta,v)|}{|\mathrm{FD}(\theta,v)|+\delta}.
\end{equation}
These terms are combined with solver instability and stiffness proxies into an empirical uncertainty score
\begin{equation}
  \widehat \epsilon_i =
  w_{\cos}D_{\cos}
  + w_{\mathrm{norm}}D_{\mathrm{norm}}
  + w_{\mathrm{FD}}R_{\mathrm{FD}}
  + w_{\mathrm{solver}}S
  + w_{\mathrm{stiff}}K,
\end{equation}
where $S$ denotes solver instability evidence and $K$ denotes stiffness evidence. The certificate assigns one of four states: \textbf{Trusted}, \textbf{Repairable}, \textbf{Unsafe}, or \textbf{Failed}. The certificate is a visible reliability estimate, not a claimed mathematical upper bound.

A \textbf{Trusted} gradient is applied without repair. A \textbf{Repairable} gradient justifies additional computation but not the original update. An \textbf{Unsafe} gradient lacks a certified descent direction after the available repairs. A \textbf{Failed} gradient contains nonfinite values, invalid shapes, or broken differentiation state. These states keep warnings, repairs, acceptances, and failures separate.

The certificate also diagnoses the failure mode. Large finite-difference residuals indicate directional inconsistency. Large pairwise disagreement indicates sensitivity-path instability. Large solver stress indicates numerical difficulty before the gradient is inspected. Event flags indicate that the smooth adjoint model may be mismatched to the trajectory. When the evidence is mixed, the diagnosis reports numerical instability instead of assigning a precise physical cause.

\subsection{Repair and Step Certification}

If the naive gradient is trusted, the optimizer can use it directly. If the certificate reports repairable risk, GradRepair-ODE switches to a more reliable gradient path. If the repaired certificate remains unsafe, the step is rejected. At step level, the decision compares the predicted gradient signal against the estimated numerical uncertainty:
\begin{equation}
  \eta \|\hat g\|^2 > \eta \epsilon \|\hat g\| + m,
\end{equation}
where $\epsilon$ is the empirical error proxy and $m$ is a small margin. This rule is sufficient for the step certificate used here; it is not a global convergence theorem.

Repair is selected by a cost-aware policy. Tolerance tightening is used when the evidence points to loose integration. Path switching is used when sensitivity routes disagree but a refined path is available. Strict recomputation is used for stiff or chaotic cases that remain repairable under additional numerical work. Rejection is selected when the trajectory contains an event discontinuity without an event-aware derivative, or when no candidate produces a positive descent margin. The policy minimizes computation subject to a reliability constraint:
\begin{equation}
  \min_{g_i\in\mathcal{C}} C_i
  \quad
  \text{such that}
  \quad
  \mathrm{Cert}(g_i)=\mathrm{Trusted}
  \quad\text{and}\quad
  \Delta_i>0.
\end{equation}
If the feasible set is empty, the update is rejected. This separates GradRepair-ODE from ordinary optimizer safeguards. Clipping and line search operate on the proposed move; GradRepair-ODE checks the numerical origin of the gradient that produced the move.

\section{Theory and Certification}

GradRepair-ODE uses a local theoretical object: a certificate for the next optimizer step. The certificate does not prove that an ordinary differential equation solver returned the exact gradient. It combines a computed gradient, an empirical uncertainty radius, and local descent geometry. Global guarantees are usually out of reach for stiff, chaotic, or event-driven dynamics. Step-level evidence is still enough to decide whether a gradient can be used, repaired, or rejected.

\subsection{Uncertainty Sets}

Let $L(\theta)$ denote the training loss, $g=\nabla L(\theta)$ the exact gradient, and $\hat g_i$ the gradient candidate produced by path $i$. GradRepair-ODE assigns the candidate an empirical uncertainty radius $\widehat\epsilon_i$ obtained from path disagreement, directional finite-difference residuals, solver instability, and stiffness evidence. The associated local uncertainty set is
\begin{equation}
  \mathcal{G}_i
  =
  \{g':\|g'-\hat g_i\|_2\leq \widehat\epsilon_i\}.
\end{equation}
This set is local. It does not describe future gradients; it describes the numerical ambiguity attached to the current candidate at the current step.

\begin{assumption}[Local smoothness]
In a neighborhood of the proposed update, $L$ is differentiable and its gradient is locally $M$-Lipschitz:
\begin{equation}
  \|\nabla L(\theta_a)-\nabla L(\theta_b)\|_2
  \leq M\|\theta_a-\theta_b\|_2 .
\end{equation}
The assumption is local to the step. It does not require global smoothness of the full dynamical system.
\end{assumption}

\begin{assumption}[Certificate calibration]
The score $\widehat\epsilon_i$ is calibrated as an empirical proxy for current-step gradient error. Larger path disagreement, larger finite-difference residuals, larger solver instability, and stronger stiffness evidence produce a larger radius. The analysis below is conditional on this radius. Weak calibration makes the certificate conservative.
\end{assumption}

\subsection{Descent Geometry}

The acceptance test is a robust first-order statement. The update $-\hat g_i$ is certified only when every gradient in $\mathcal{G}_i$ gives a negative directional derivative along that update with margin $m\geq 0$:
\begin{equation}
  \sup_{g'\in\mathcal{G}_i}\langle g',-\hat g_i\rangle < -m.
  \label{eq:robust-descent}
\end{equation}

\begin{lemma}[Support function of the gradient error ball]
For $\mathcal{G}_i=\{g':\|g'-\hat g_i\|_2\leq \widehat\epsilon_i\}$,
\begin{equation}
  \sup_{g'\in\mathcal{G}_i}\langle g',-\hat g_i\rangle
  =
  -\|\hat g_i\|_2^2+\widehat\epsilon_i\|\hat g_i\|_2 .
\end{equation}
\end{lemma}

\begin{proof}
Write $g'=\hat g_i+e$ with $\|e\|_2\leq \widehat\epsilon_i$. Then
\begin{equation}
  \langle g',-\hat g_i\rangle
  =
  -\|\hat g_i\|_2^2-\langle e,\hat g_i\rangle .
\end{equation}
The supremum occurs when $e$ is anti-parallel to $\hat g_i$, giving $-\langle e,\hat g_i\rangle=\widehat\epsilon_i\|\hat g_i\|_2$.
\end{proof}

\begin{proposition}[Certified descent margin]
Under the uncertainty set above, condition~\eqref{eq:robust-descent} is equivalent to
\begin{equation}
  \Delta_i
  =
  \|\hat g_i\|_2^2
  -
  \widehat\epsilon_i\|\hat g_i\|_2
  -
  m
  >0 .
  \label{eq:descent-margin}
\end{equation}
If Assumption 1 holds, then for a learning rate $\eta>0$,
\begin{equation}
  L(\theta-\eta\hat g_i)
  \leq
  L(\theta)
  -
  \eta\left(\|\hat g_i\|_2^2-\widehat\epsilon_i\|\hat g_i\|_2\right)
  +
  \frac{M\eta^2}{2}\|\hat g_i\|_2^2 .
\end{equation}
Thus a positive margin yields a descent certificate for sufficiently small $\eta$.
\end{proposition}

\begin{corollary}[Rejection is informative]
If $\|\hat g_i\|_2\leq \widehat\epsilon_i$, the uncertainty ball contains gradients whose projection along $-\hat g_i$ is nonnegative. In that case, rejecting the update is not a failed optimization step; it is a recorded absence of descent evidence.
\end{corollary}

\begin{proposition}[Monotonicity of evidence]
For a fixed candidate norm $\|\hat g_i\|_2$ and margin $m$, the descent margin $\Delta_i$ is strictly decreasing in the uncertainty radius $\widehat\epsilon_i$ whenever $\|\hat g_i\|_2>0$. Consequently, adding evidence that increases numerical uncertainty can only make the certificate more conservative, never more permissive.
\end{proposition}

\begin{proof}
Equation~\eqref{eq:descent-margin} gives
\begin{equation}
  \frac{\partial \Delta_i}{\partial \widehat\epsilon_i}
  =
  -\|\hat g_i\|_2 .
\end{equation}
The derivative is negative for every nonzero candidate gradient.
\end{proof}

\begin{proposition}[Descent cone]
If the true error satisfies $\|g-\hat g_i\|_2\leq\epsilon<\|\hat g_i\|_2$, then every admissible true gradient lies in a cone around $\hat g_i$ with half-angle $\alpha$ satisfying
\begin{equation}
  \sin\alpha \leq \frac{\epsilon}{\|\hat g_i\|_2}.
\end{equation}
The certificate evaluates direction as well as magnitude. A low-norm candidate can be rejected even with small absolute error, while a high-norm candidate can remain certified under larger numerical uncertainty.
\end{proposition}

\subsection{Directional Evidence}

Finite differences are not treated as ground truth. They are local directional witnesses. For $v$ with $\|v\|_2=1$, define
\begin{equation}
  D_h(\theta;v)
  =
  \frac{L(\theta+h v)-L(\theta-h v)}{2h}.
\end{equation}

\begin{proposition}[Centered directional consistency]
If $\phi(\alpha)=L(\theta+\alpha v)$ is three times continuously differentiable near $0$, then
\begin{equation}
  D_h(\theta;v)
  =
  \langle \nabla L(\theta),v\rangle
  +
  O(h^2)
  +
  O(u/h),
\end{equation}
where $u$ denotes the floating-point roundoff scale. Therefore, a sign disagreement between $\langle \hat g_i,v\rangle$ and $D_h(\theta;v)$ is meaningful only when $h$ is chosen outside both the truncation-dominated and roundoff-dominated regimes.
\end{proposition}

\begin{corollary}[Multi-direction directional risk]
For independent directions $v_1,\ldots,v_k$, define
\begin{equation}
  A_k(\hat g_i)=
  \frac{1}{k}\sum_{j=1}^k
  \mathbf{1}\!\left[
  \operatorname{sign}\langle \hat g_i,v_j\rangle
  =
  \operatorname{sign}D_h(\theta;v_j)
  \right].
\end{equation}
Low $A_k$ is direct evidence that the candidate is directionally fragile. It is evidence rather than proof, but it is more informative than accepting a gradient only because automatic differentiation returned a finite tensor.
\end{corollary}

\begin{proposition}[Directional rejection witness]
Assume there exists a tested direction $v$ such that the centered finite difference is in its consistency regime and
\begin{equation}
  \operatorname{sign}\langle \hat g_i,v\rangle
  \neq
  \operatorname{sign}D_h(\theta;v),
  \qquad
  |D_h(\theta;v)|>\rho ,
\end{equation}
for a tolerance $\rho$ larger than the estimated finite-difference noise floor. Then the candidate cannot be certified by directional evidence alone; any acceptance must be justified by stronger path or solver evidence. GradRepair-ODE uses this witness to move the candidate from trusted to repairable or unsafe.
\end{proposition}

\begin{proof}
In the finite-difference consistency regime, $D_h(\theta;v)$ estimates $\langle g,v\rangle$ up to truncation and roundoff errors below $\rho$. A sign mismatch above that floor means the candidate and the local directional witness disagree on whether moving along $v$ increases or decreases the loss. The candidate may still be correct if the witness is corrupted, but the available evidence is no longer sufficient for trust.
\end{proof}

\subsection{Repair as Constrained Selection}

Let $\mathcal{I}$ be the available set of gradient paths and let $C_i$ denote the cost of candidate $i$. GradRepair-ODE selects
\begin{equation}
  i^\star
  \in
  \arg\min_{i\in\mathcal{I}} C_i
  \quad
  \text{subject to}
  \quad
  \Delta_i>0,\quad A_k(\hat g_i)\geq \tau .
\end{equation}
If no candidate satisfies the constraints, the optimizer step is rejected. Repair is constrained selection under numerical evidence, not an unconditional move toward the strictest computation.

\begin{proposition}[Selective repair cost]
Let the naive path cost be $C_0$, diagnostic cost $C_d$, repair cost $C_r$, and always-strict cost $C_s$, all measured in function evaluations. If a fraction $p$ of steps trigger repair, the expected GradRepair cost is
\begin{equation}
  C_{\mathrm{GR}}=C_0+C_d+pC_r.
\end{equation}
An always-strict strategy is more expensive whenever
\begin{equation}
  p < \frac{C_s-C_0-C_d}{C_r}.
\end{equation}
This gives the cost-reliability premise of GradRepair-ODE: use cheap gradients when the evidence is coherent, pay for repair when evidence breaks, and reject when no path produces a certified step.
\end{proposition}

\begin{proposition}[Dominance over warning-only diagnostics]
Consider two policies that observe the same certificate. A warning-only policy always applies the original candidate, while a certificate-controlled policy applies only candidates satisfying $\Delta_i>0$ and rejects otherwise. On any step where the original candidate has $\Delta_i\leq 0$, the warning-only policy accepts an uncertified update and the certificate-controlled policy does not. Therefore, for a fixed sequence of certificates, certificate control weakly reduces the number of uncertified accepted steps relative to warning-only diagnostics.
\end{proposition}

\begin{proof}
The statement follows directly from the policy definitions. Steps with $\Delta_i>0$ may be accepted by both policies. Steps with $\Delta_i\leq 0$ are accepted by the warning-only policy but withheld by the certificate-controlled policy. Summing over the sequence gives the weak reduction.
\end{proof}

\begin{corollary}[Irreparable steps]
For discontinuous or hybrid dynamics, an ordinary smooth-gradient path may remain uncertified even after recomputation. Preventing an unsupported update is a valid certified action.
\end{corollary}

\subsection{What the Certificate Establishes}

The certificate establishes a step-local statement:
\begin{equation}
  \text{certified}
  \quad\Longleftrightarrow\quad
  \exists i\in\mathcal{I}
  \text{ such that }
  \Delta_i>0
  \text{ and }
  A_k(\hat g_i)\geq\tau .
\end{equation}
It does not establish global convergence, exact adjoint correctness, or universal solver reliability. This narrowness is a strength of the formulation. The object being certified is precisely the object consumed by the optimizer: the next update direction. When the candidate fails the test, the output is not an ambiguous warning but an optimizer action: repair with a different path, reduce trust in the step, or reject the update.

The reliability claim can fail at three levels. At the gradient level, path disagreement and finite-difference evidence can contradict the candidate. At the step level, the descent margin can fail. At the system level, the cost of repair can exceed the value of using a cheap path. The certification boundary is local, numerical, optimizer-facing, and still strong enough to change training behavior.

Appendix~A connects the diagnostics to standard inequalities. Pairwise path disagreement lower-bounds the error of at least one candidate. Finite-difference residuals lower-bound directional projection error after accounting for the finite-difference noise floor. Sensitivity growth follows the variational equation through a Gronwall-type amplification bound. These results do not turn the empirical proxy into an oracle; they explain why the proxy uses these numerical quantities.

\section{Experiments}

\subsection{Systems and Baselines}

We evaluate six synthetic systems: a harmonic oscillator, Van der Pol oscillator, Robertson chemical kinetics, Lorenz dynamics, a bouncing ball with event discontinuities, and a neural right-hand side (RHS) system. The systems cover smooth dynamics, stiffness, chaotic sensitivity, hybrid events, and a small neural vector field.

The design stresses the gradient interface. It includes cases where a coarse sensitivity path is harmless, cases where it is directionally wrong, and cases where no smooth-gradient update is defensible. We compare a naive coarse solver path, a discrete finer solver path, a checkpointed-style repair path, and GradRepair-ODE. The strict reference is used only to measure agreement. The same deterministic protocol generates all reported numbers.

Each system isolates a numerical mechanism that a scientific-computing reader can inspect. Smooth periodic dynamics should not trigger repair. Stiffness should expose tolerance and path sensitivity. Chaotic dynamics should challenge long-horizon directionality. Event discontinuities should prevent an ordinary smooth-gradient certificate from making an unsupported claim. The neural RHS checks that the reliability layer stays quiet when the learned vector field is numerically benign. The benchmark tests mechanism recognition as well as final loss.

The baselines address different objections. The naive solver path tests the common training interface in which a returned gradient is applied without further numerical evidence. Gradient clipping tests whether magnitude control is enough. Armijo-style line search tests whether sufficient decrease on the proposed move can substitute for gradient-source repair. GradRepair-ODE faces a stricter test: it must reduce unsafe accepted updates, preserve benign updates, repair cases with coherent numerical alternatives, and reject cases where smooth-gradient evidence is inadequate.

\subsection{Deterministic Protocol}

All tables and figures come from the same recorded diagnostic run. The random seed fixes the initial conditions, parameter perturbations, and finite-difference directions. For each system, the experiment records the loss, gradient candidates, cosine similarity to the strict reference, finite-difference residuals, failure diagnosis, repair action, acceptance decision, wall-clock time, and function-evaluation counts. Training-level summaries are computed from the saved step records. The ablation table replays the same evidence, so changes across rows reflect decision logic rather than a different numerical run.

The protocol centers the optimizer decision instead of a post-hoc curve. Repair without cost accounting is incomplete. Detection without control is incomplete. Rejecting every suspicious step is safe but uninformative. The recorded protocol keeps those cases separate.

\subsection{Gradient Reliability}

Table~\ref{tab:gradient} shows the reliability pattern. Harmonic, Van der Pol, and neural RHS dynamics have cosine agreement close to one and no repair. Robertson and bouncing-ball gradients reverse direction under the naive path, with cosine similarity $-1.000$ against the strict reference. Lorenz does not fully reverse, but its naive cosine similarity drops to $0.698$. GradRepair restores Robertson and Lorenz to cosine similarity $1.000$ and refuses the event-discontinuous bouncing-ball update, where the certificate cannot justify a smooth-gradient claim.

\begin{table}[H]
  \centering
  \caption{Gradient reliability across ODE systems. Gray cells mark quantities that triggered certificate attention: large finite-difference risk or a nonzero repair rate. A dash indicates that no repaired gradient was accepted. Cosine similarity is measured against a strict reference gradient. Finite-difference risk is $\log_{10}(1+\mathrm{finite\mbox{-}difference\ residual})$, with the raw residual preserved in the experiment log.}
  \label{tab:gradient}
  \begin{tabular}{llllll}\toprule
\rowcolor{gray!10} System & Naive cos & Discrete cos & GradRepair cos & Finite-difference risk & Repair rate \\
\midrule
Harmonic & 1.000 & 1.000 & 1.000 & 0.003 & 0.00 \\
Van der Pol & 1.000 & 1.000 & 1.000 & 0.064 & 0.00 \\
Robertson & -1.000 & -1.000 & 1.000 & \cellcolor{gray!12}8.140 & \cellcolor{gray!12}1.00 \\
Lorenz & 0.698 & 1.000 & 1.000 & 0.228 & \cellcolor{gray!12}1.00 \\
Bouncing ball & -1.000 & -1.000 & -- & \cellcolor{gray!12}0.934 & \cellcolor{gray!12}1.00 \\
Neural RHS & 1.000 & 1.000 & 1.000 & 0.001 & 0.00 \\
\bottomrule
\end{tabular}

\end{table}

\begin{figure}[H]
  \centering
  \includegraphics[width=\linewidth]{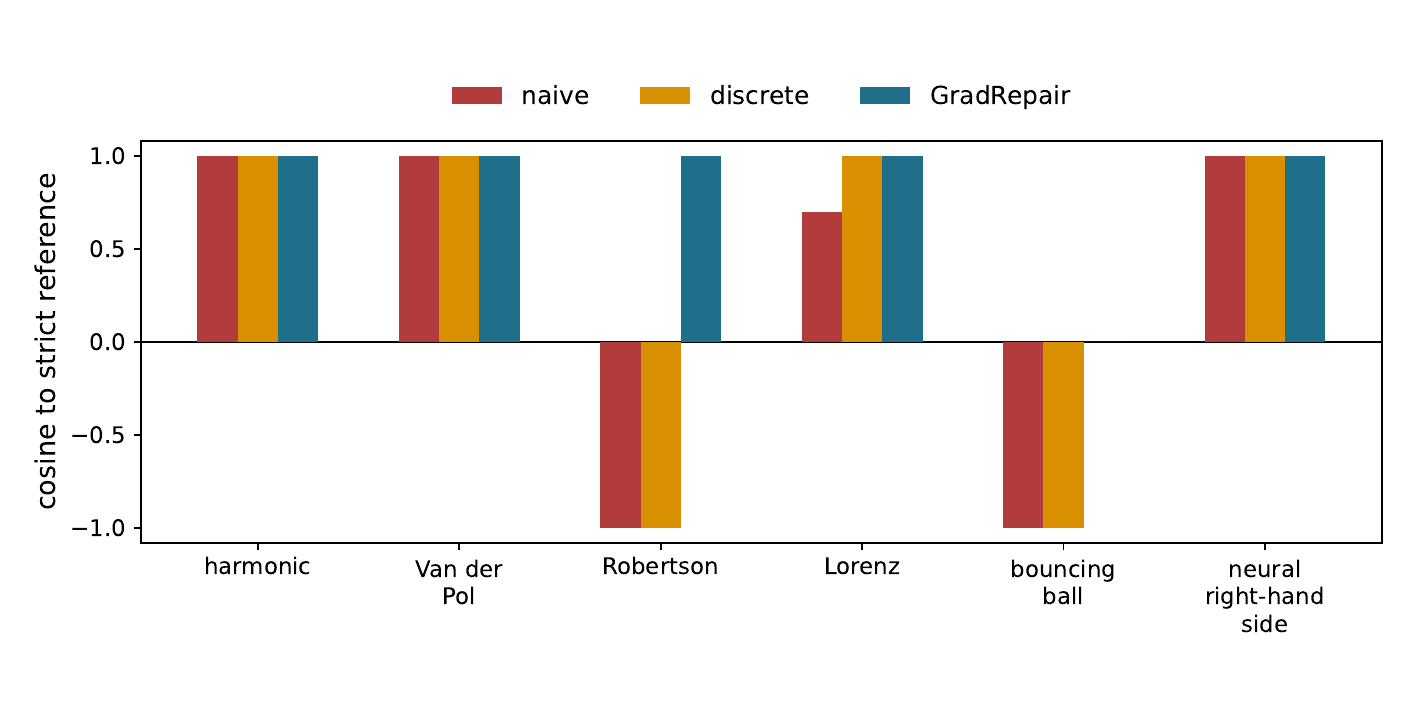}
  \caption{Gradient direction agreement against a strict reference. GradRepair-ODE leaves smooth low-risk systems unchanged, while exposing failure or repair needs in stiff, chaotic, and event-driven systems.}
  \label{fig:direction}
\end{figure}

\begin{table}[H]
  \centering
  \caption{Failure routing induced by the certificate. Gray cells mark final accepted or rejected interventions selected by the certificate. The table reports the primary numerical evidence, selected repair action, final optimizer decision, and function-evaluation multiplier for each system.}
  \label{tab:routing}
  \begin{tabular}{llllr}\toprule
\rowcolor{gray!10} System & Primary evidence & Routing action & Final step & Cost multiplier \\
\midrule
Harmonic & consistent paths & none & accepted & 1.0 \\
Van der Pol & consistent paths & none & accepted & 1.0 \\
Robertson & finite-difference residual & strict recomputation & \cellcolor{gray!12}accepted & 301.0 \\
Lorenz & path disagreement & checkpointed routing & \cellcolor{gray!12}accepted & 9.3 \\
Bouncing ball & finite-difference residual & reject & \cellcolor{gray!12}rejected & 41.0 \\
Neural RHS & consistent paths & none & accepted & 1.0 \\
\bottomrule
\end{tabular}

\end{table}

\subsection{Training Stability}

Table~\ref{tab:training} reports the optimizer-facing consequence. The naive loop accepts 37 uncertified updates across Robertson, Lorenz, and bouncing ball. GradRepair reduces that count to zero. The method is not conservative by default: low-risk systems are trusted, Lorenz and Robertson are repaired, and the event-discontinuous system is rejected. The intervention occurs where the evidence changes the status of the step.

\begin{table}[H]
  \centering
  \caption{Training stability summary. Red cells mark unsafe updates accepted by the naive loop, green cells mark repaired or cleared unsafe counts, yellow cells mark certified rejection, and gray cells mark trusted no-intervention decisions. Unsafe accepted counts steps that the naive loop applied despite a non-trusted certificate.}
  \label{tab:training}
  \begin{tabular}{lrrrrrrl}\toprule
\rowcolor{gray!10} System & \shortstack{Naive\\loss} & \shortstack{GradRepair\\loss} & \shortstack{Naive\\unsafe} & \shortstack{GradRepair\\unsafe} & \shortstack{GradRepair\\rejected} & \shortstack{Repair\\rate} & Decision \\
\midrule
Harmonic & 1.273e-05 & 1.273e-05 & 0 & 0 & 0 & 0.00 & \cellcolor{gray!10}trust \\
Van der Pol & 3.081 & 3.081 & 0 & 0 & 0 & 0.00 & \cellcolor{gray!10}trust \\
Robertson & 0.03433 & 0.03483 & \cellcolor{red!16}18 & \cellcolor{green!12}0 & 0 & \cellcolor{green!12}1.00 & \cellcolor{green!14}repair \\
Lorenz & 48.42 & 48.33 & \cellcolor{red!16}18 & \cellcolor{green!12}0 & 0 & \cellcolor{green!12}1.00 & \cellcolor{green!14}repair \\
Bouncing ball & 6.408 & 1.068 & \cellcolor{red!16}1 & \cellcolor{green!12}0 & \cellcolor{yellow!18}18 & \cellcolor{green!12}1.00 & \cellcolor{yellow!18}reject \\
Neural RHS & 0.09854 & 0.09854 & 0 & 0 & 0 & 0.00 & \cellcolor{gray!10}trust \\
\bottomrule
\end{tabular}

\end{table}

\begin{figure}[H]
  \centering
  \includegraphics[width=\linewidth]{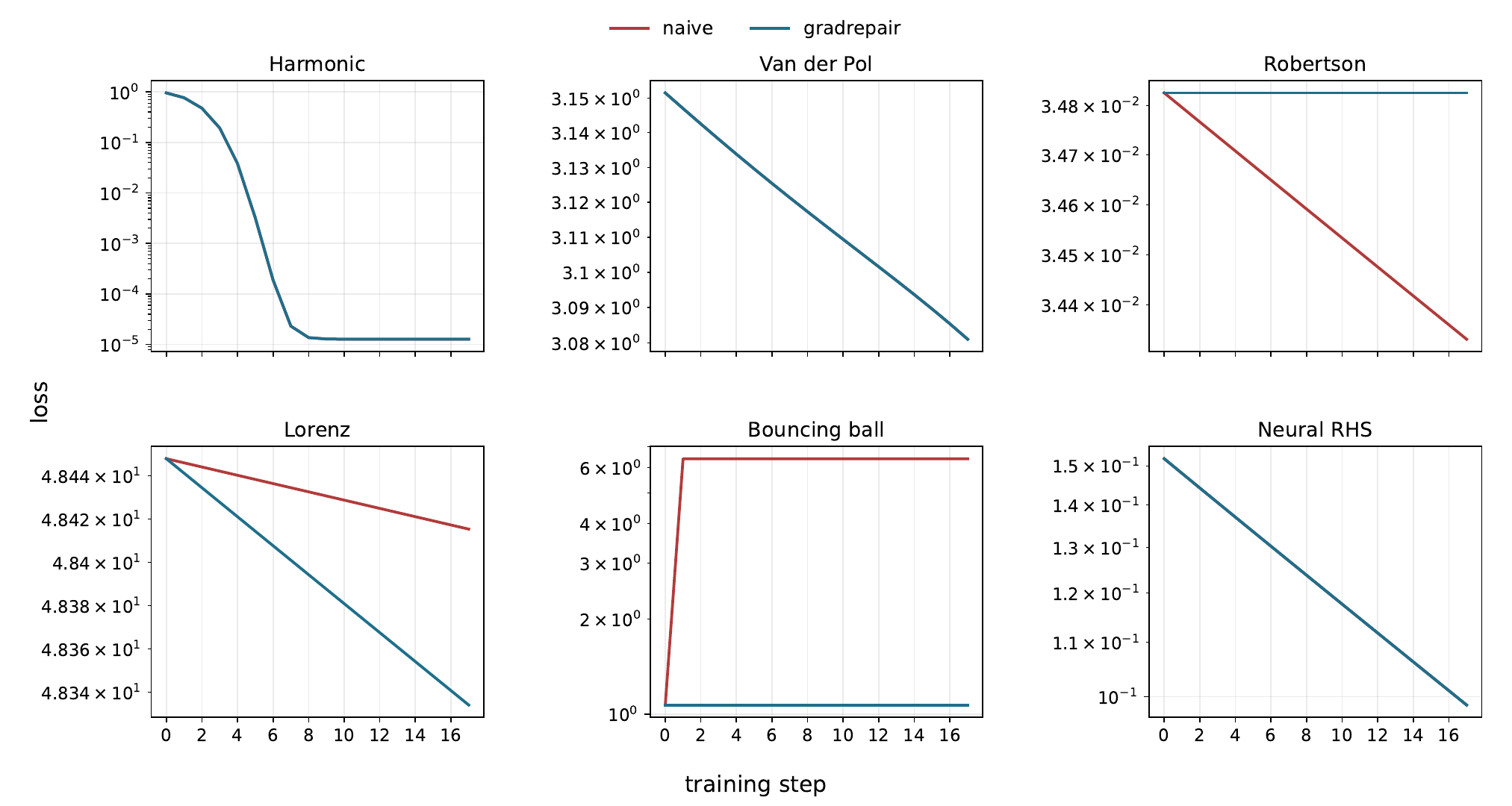}
  \caption{Loss trajectories under naive training and GradRepair-ODE. The goal is reliability control rather than aggressive optimization; rejected steps preserve evidence when no safe gradient is available.}
  \label{fig:loss}
\end{figure}

\subsection{Ablation and Sensitivity}

Table~\ref{tab:ablation} isolates the reliability components. The ablation asks how many unsafe updates remain when detection, repair routing, finite-difference evidence, or step certification is removed. Detection without action leaves all unsafe accepted updates in place. Removing step certification allows the event-discontinuous case to pass after a failed repair. Removing repair routing avoids unsafe updates but rejects every non-trusted step, including repairable Robertson and Lorenz updates. The full system repairs repairable steps, rejects the irreparable event case, and leaves zero unsafe accepted updates.

\begin{table}[H]
  \centering
  \caption{Ablation study from the same recorded step evidence. Red cells mark unsafe accepted updates, green cells mark repaired or cleared unsafe counts, and yellow cells mark rejected updates.}
  \label{tab:ablation}
  \begin{tabular}{lrrrl}\toprule
\rowcolor{gray!10} Policy & Unsafe accepted & Repaired & Rejected & Main effect \\
\midrule
naive & \cellcolor{red!16}37 & 0 & 0 & no control \\
detect only & \cellcolor{red!16}37 & 0 & 0 & warning only \\
no finite-difference evidence & \cellcolor{red!16}18 & \cellcolor{green!12}36 & 0 & event risk passes \\
no repair routing & \cellcolor{green!12}0 & 0 & \cellcolor{yellow!18}54 & over-rejects \\
no step certification & \cellcolor{red!16}18 & \cellcolor{green!12}54 & 0 & applies failed repair \\
full GradRepair-ODE & \cellcolor{green!12}0 & \cellcolor{green!12}54 & \cellcolor{yellow!18}18 & repair or reject \\
\bottomrule
\end{tabular}

\end{table}

\begin{figure}[H]
  \centering
  \includegraphics[width=0.95\linewidth]{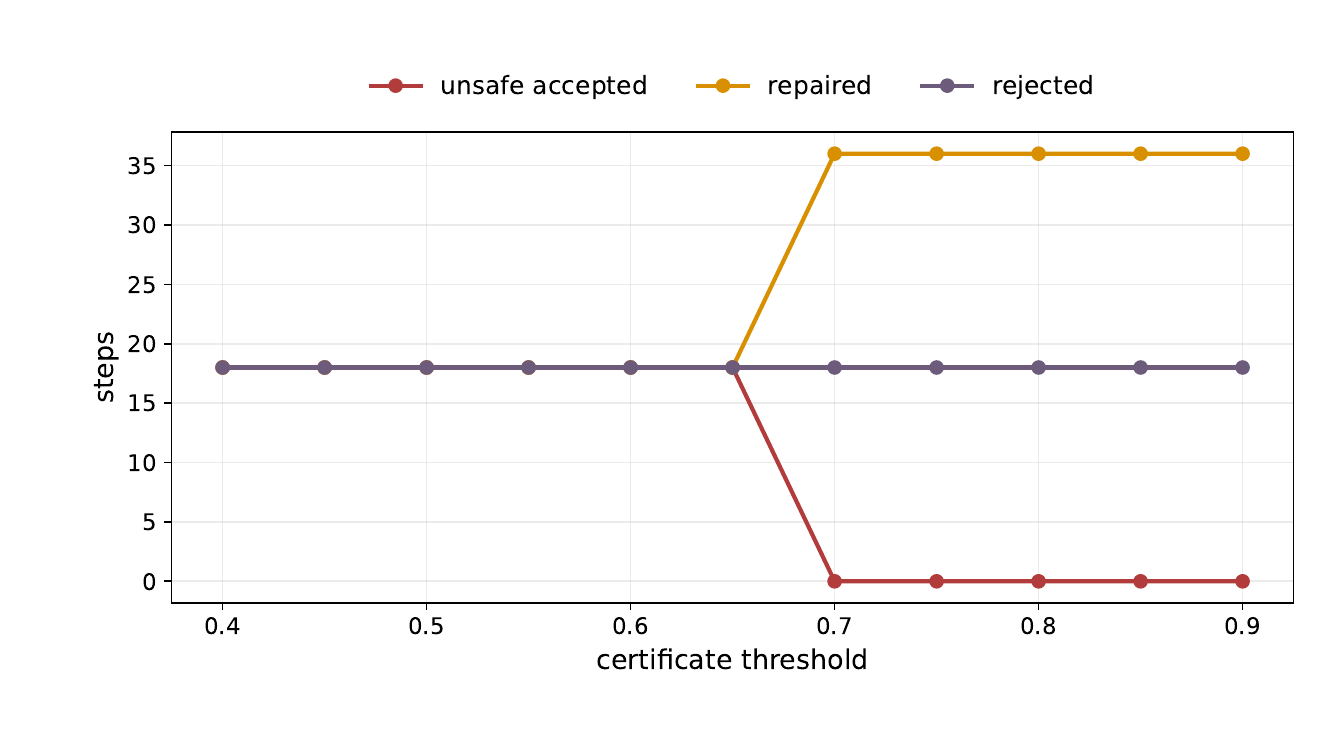}
  \caption{Threshold sensitivity for the certificate decision. The reliability layer has a stable operating region: stricter thresholds increase repair and rejection, while unsafe accepted steps disappear once suspicious gradients are no longer trusted by default.}
  \label{fig:threshold}
\end{figure}

\begin{figure}[H]
  \centering
  \includegraphics[width=0.95\linewidth]{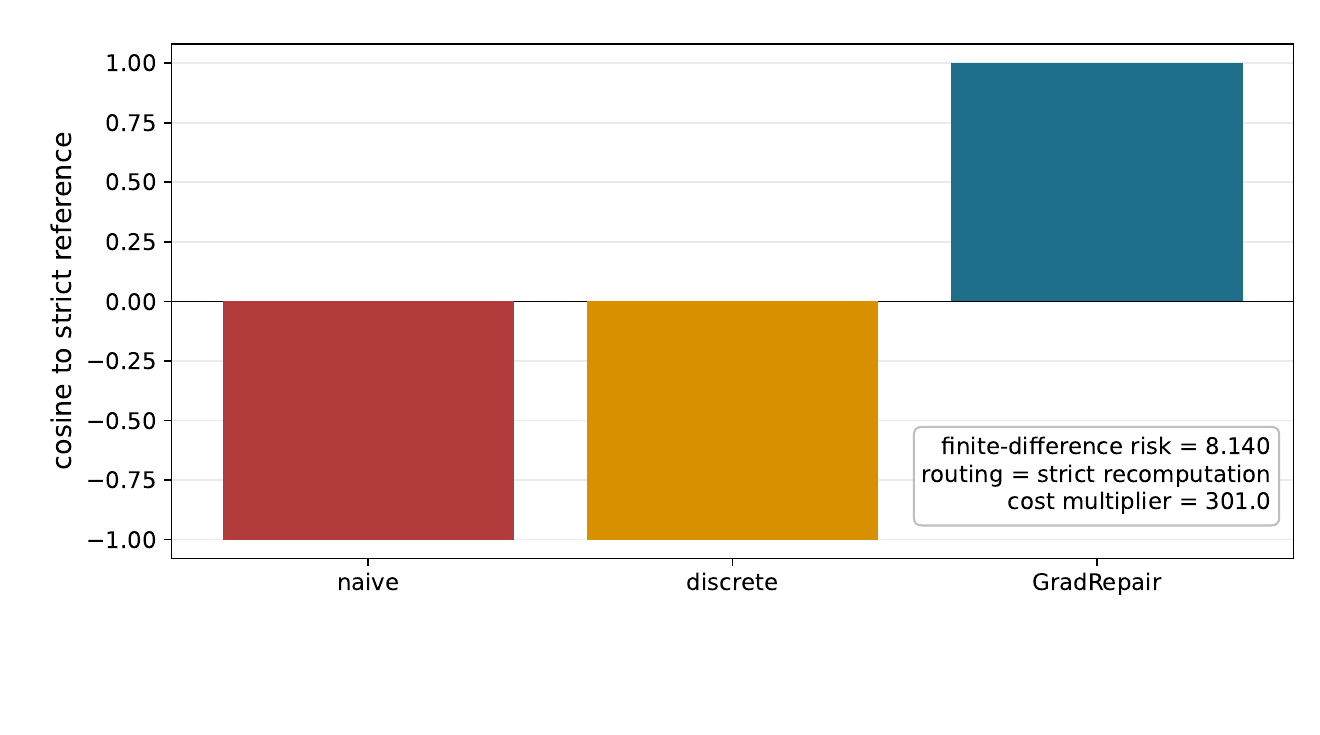}
  \caption{Robertson case study. The naive and discrete paths point in the opposite direction from the strict reference, while GradRepair-ODE routes to strict recomputation and restores cosine similarity to $1.000$.}
  \label{fig:case}
\end{figure}

\subsection{Cost and Classical Safeguards}

Table~\ref{tab:cost} reports wall-clock time and function-evaluation overhead for the diagnostic step. Low-risk systems incur essentially no repair overhead, Lorenz is repaired at moderate cost, and Robertson shows the expensive edge of strict recomputation. GradRepair-ODE does not claim that certification is free. It makes the price of reliable gradients explicit and pays it only where the evidence requires it.

\begin{table}[H]
  \centering
  \caption{Measured diagnostic cost. Yellow cells mark time or function-evaluation multipliers above $5\times$. Wall-clock measurements are generated by the same deterministic experiment script and reported in milliseconds.}
  \label{tab:cost}
  \begin{tabular}{lrrrr}\toprule
\rowcolor{gray!10} System & Naive ms & GradRepair ms & Time multiplier & NFE multiplier \\
\midrule
Harmonic & 4.48 & 4.48 & 1.0 & 1.0 \\
Van der Pol & 0.54 & 0.54 & 1.0 & 1.0 \\
Robertson & 0.34 & 29.04 & \cellcolor{yellow!18}85.9 & \cellcolor{yellow!18}301.0 \\
Lorenz & 2.10 & 17.74 & \cellcolor{yellow!18}8.4 & \cellcolor{yellow!18}9.3 \\
Bouncing ball & 0.23 & 2.74 & \cellcolor{yellow!18}11.9 & \cellcolor{yellow!18}41.0 \\
Neural RHS & 2.69 & 2.69 & 1.0 & 1.0 \\
\bottomrule
\end{tabular}

\end{table}

Table~\ref{tab:safeguards} compares GradRepair-ODE with two standard optimizer safeguards. Gradient clipping controls magnitude but cannot repair a wrong direction \cite{pascanu2013difficulty}. Line search checks whether a proposed move decreases the measured loss \cite{armijo1966minimization,nocedal2006optimization}, but it does not identify the numerical source of the gradient error and may reject without producing a reliable replacement direction. GradRepair-ODE addresses a different failure mode: it uses numerical evidence to decide whether to accept, repair, or reject the gradient before the optimizer consumes it.

\begin{table}[H]
  \centering
  \caption{Comparison with classical safeguards. Red cells mark unsafe accepted updates, green cells mark cleared unsafe counts, yellow cells mark rejected updates, and light red cells mark loss spikes.}
  \label{tab:safeguards}
  \begin{tabular}{llrrr}\toprule
\rowcolor{gray!10} Method & Mechanism & Unsafe accepted & Rejected & Loss spikes \\
\midrule
naive & coarse update & \cellcolor{red!16}37 & 0 & \cellcolor{red!12}1 \\
naive + clipping & norm control only & \cellcolor{red!16}47 & 0 & \cellcolor{red!12}3 \\
naive + line search & loss decrease check & \cellcolor{red!16}37 & \cellcolor{yellow!18}17 & 0 \\
GradRepair-ODE & repair or reject & \cellcolor{green!12}0 & \cellcolor{yellow!18}18 & 0 \\
\bottomrule
\end{tabular}

\end{table}

\begin{figure}[H]
  \centering
  \includegraphics[width=0.95\linewidth]{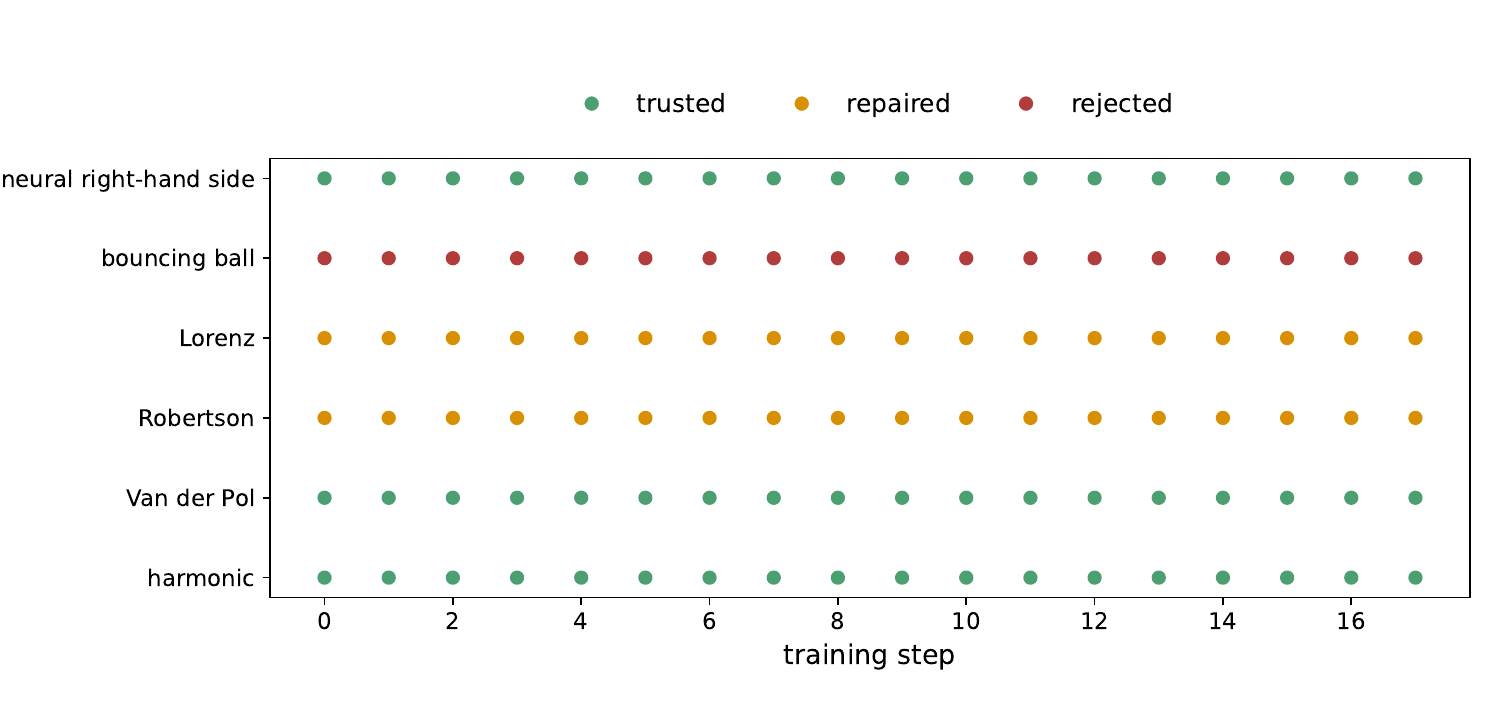}
  \caption{Per-step GradRepair decisions. Trusted steps are left unchanged, repaired steps are routed to a safer gradient path, and rejected steps are withheld when no smooth-gradient claim is certified.}
  \label{fig:timeline}
\end{figure}

\begin{figure}[H]
  \centering
  \includegraphics[width=0.95\linewidth]{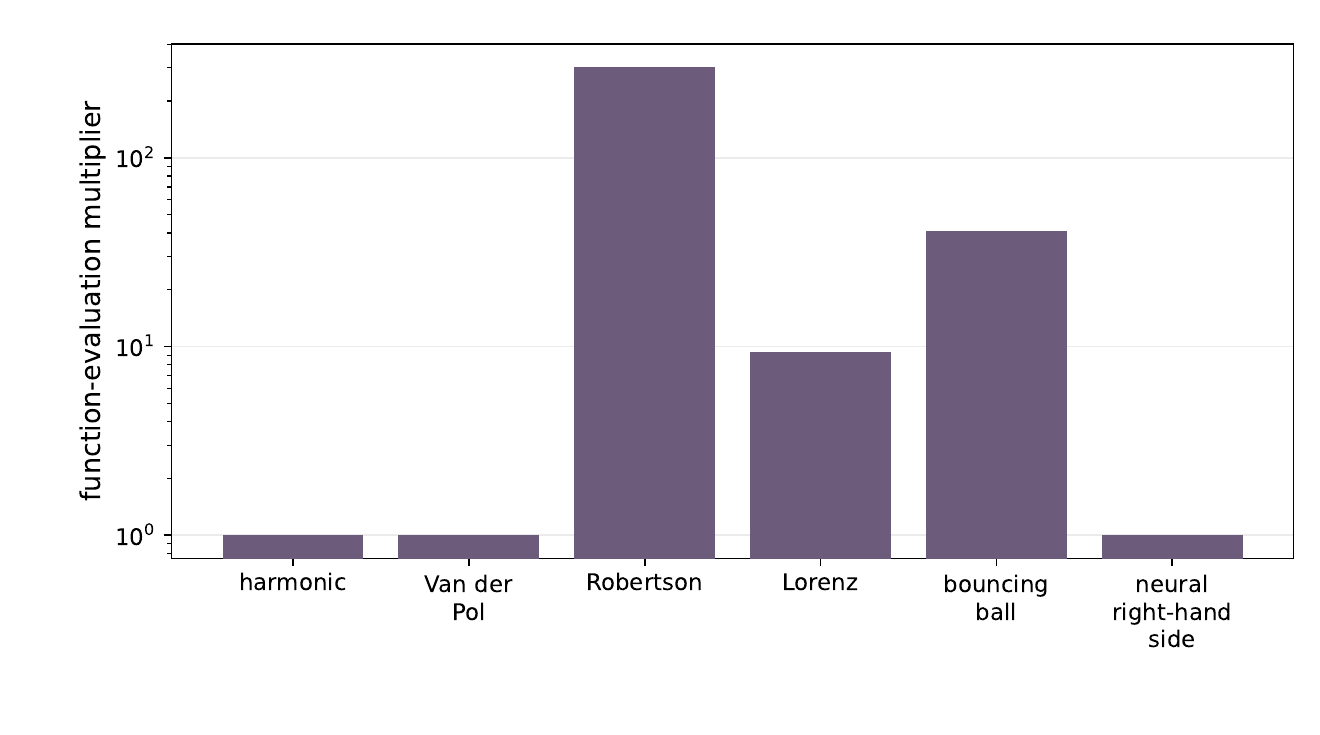}
  \caption{Repair cost measured as the multiplier in function evaluations relative to the naive gradient path, shown on a logarithmic scale to preserve the full Robertson cost.}
  \label{fig:cost}
\end{figure}

\begin{figure}[H]
  \centering
  \includegraphics[width=0.95\linewidth]{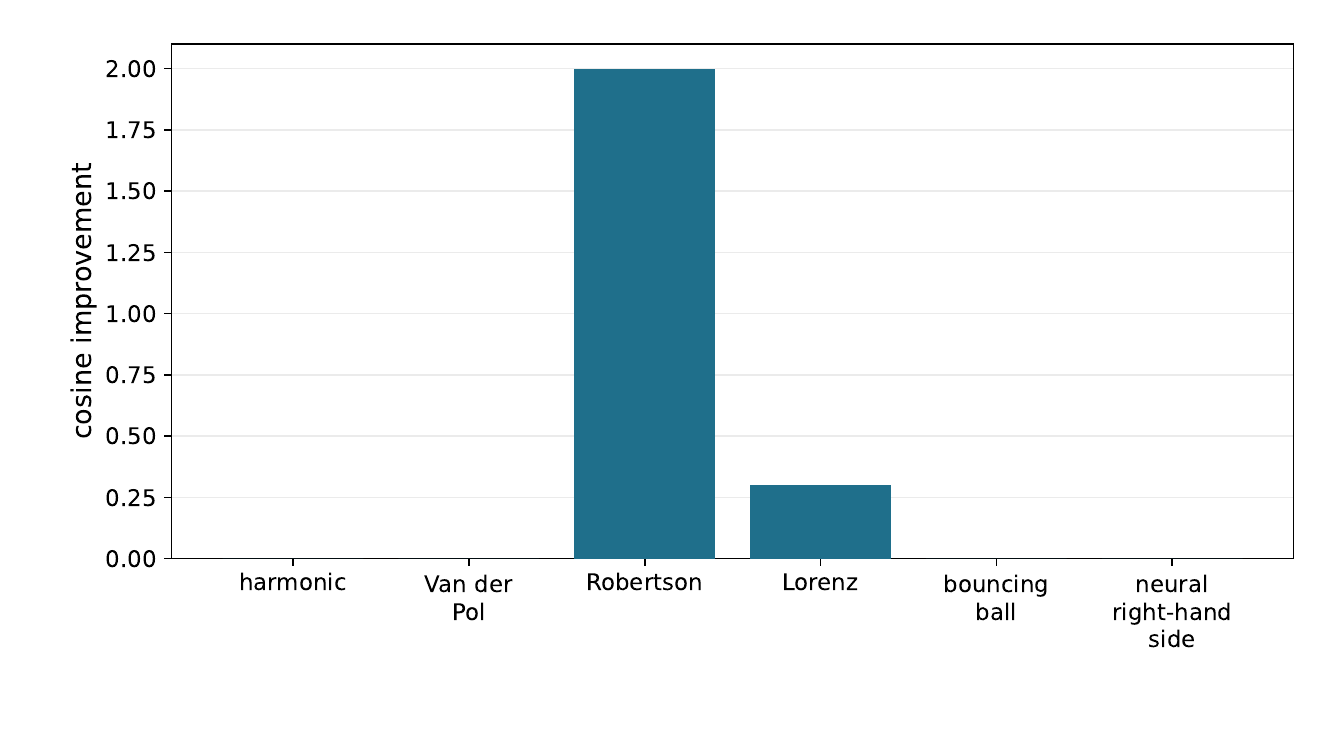}
  \caption{Reliability gain measured as cosine improvement relative to the naive gradient path. The largest gains occur where the naive direction is inconsistent with the strict reference.}
  \label{fig:gain}
\end{figure}

\begin{figure}[H]
  \centering
  \includegraphics[width=0.95\linewidth]{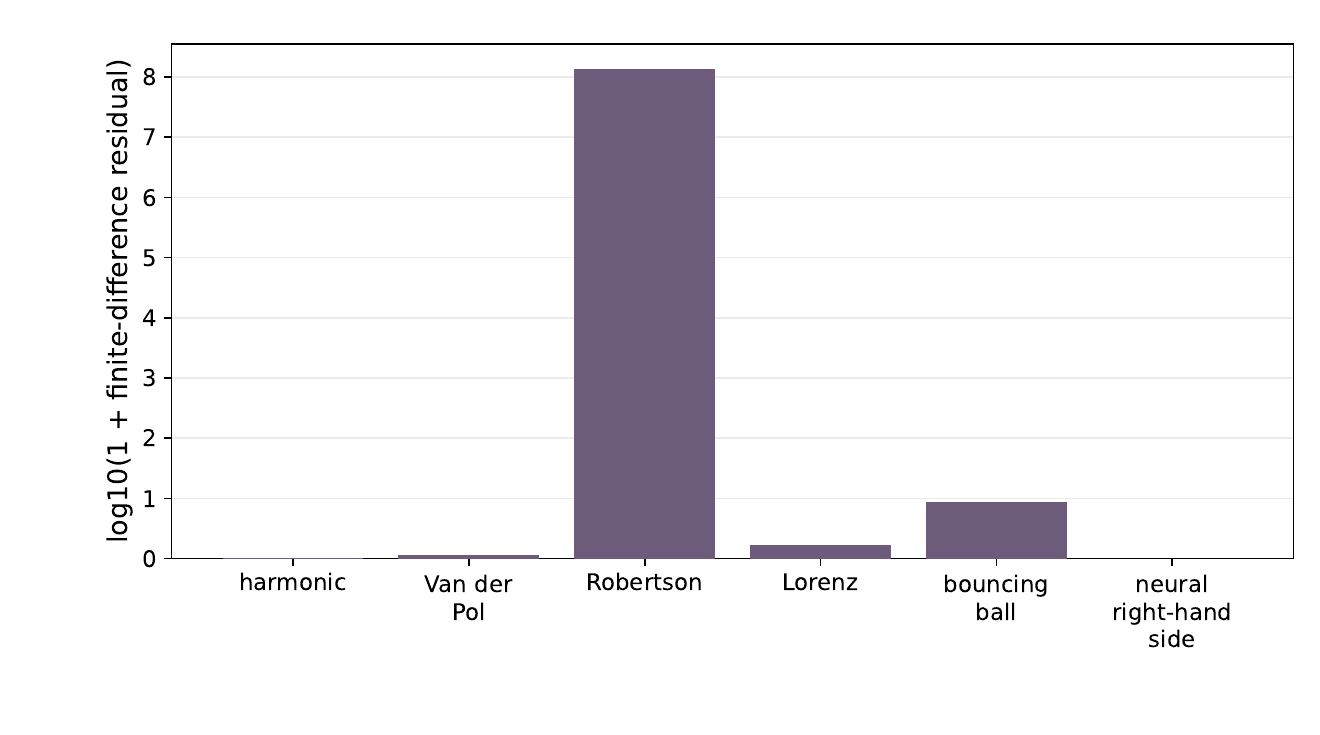}
  \caption{Directional finite-difference residuals. Large residuals identify candidates whose directional derivative evidence is inconsistent with the proposed gradient.}
  \label{fig:fd}
\end{figure}

\section{Discussion}

The experiments show that gradient reliability can be handled at the optimizer boundary. GradRepair-ODE changes the update stream: low-risk harmonic and neural RHS systems pass without intervention, Robertson and Lorenz are repaired to agreement with a strict reference, and the event-discontinuous system is rejected when the smooth-gradient evidence is insufficient. The repair-controlled loop eliminates unsafe accepted steps in these runs, from $37$ under the naive policy to $0$ under GradRepair-ODE.

The design separates detection, diagnosis, repair, and step acceptance. A scalar loss curve rarely reveals whether instability comes from the model, the solver, the sensitivity path, or an event discontinuity. GradRepair-ODE records finite-difference residuals, pairwise gradient disagreement, solver failure, repair action, and step acceptance for each update.

This is more precise than a solver-tolerance recommendation. Tight tolerances can be wasteful when the gradient is already coherent, and they do not fix a wrong differentiability assumption, as the event-discontinuous case shows. GradRepair-ODE asks what evidence supports this step. That question gives the method its cost-reliability profile: harmonic, Van der Pol, and neural RHS systems do not pay for repair; Robertson and Lorenz do; the bouncing-ball system is not allowed to turn a discontinuity into a smooth-gradient update.

The comparison with clipping and line search clarifies the gap. Clipping modifies vector length; it does not repair direction. Line search can reject a bad move, but it does not say whether the failure came from stiffness, chaos, interpolation, event structure, or sensitivity-path mismatch. GradRepair-ODE answers the earlier numerical question: whether the vector used to define the move deserves trust.

The claim has a clear boundary. Finite differences are directional checks with truncation and floating-point error. The certificate is an empirical proxy, not a formal upper bound. Within that boundary, the experiments show that a small evidence layer can distinguish trusted, repairable, unsafe, and failed gradient states while preserving the option to refuse a step when repair would overclaim.

For scientific computing, the practical message is direct. When a differentiable solver supplies a gradient to an optimizer, it should also supply enough numerical evidence to make the update accountable. GradRepair-ODE gives one design for that evidence: multiple gradient candidates, calibrated disagreement, repair routing, and a descent certificate attached to the step itself.

\section{Limitations and Scope}

GradRepair-ODE is a reliability layer for optimizer steps, not a universal solver correctness theorem. The experiments are synthetic and diagnostic; they expose numerical failure modes under controlled conditions rather than establish application-level superiority. That scope matches the claim: the paper studies whether unreliable ODE gradients can be detected, repaired, or rejected before they become optimizer updates.

Several boundaries are visible in the formulation. Directional finite-difference checks become expensive in very high-dimensional parameter spaces. Random directional checks reduce the cost, but they do not replace full-gradient verification. Stochastic objectives or noisy mini-batches require additional calibration because finite-difference residuals can mix numerical error with sampling noise. Hybrid systems with event-time derivatives may require event-aware adjoints; the bouncing-ball experiment is therefore rejected instead of repaired. Strict recomputation can be expensive, as Robertson shows. The method reports this cost so reliability and computation can be compared directly.

The controlled benchmark suite should be read as a mechanism study. Six systems do not cover all scientific machine learning. They cover the failure modes most likely to corrupt ODE gradients: benign smooth flow, stiffness, chaos, event discontinuity, and neural-vector-field training. A domain application would add modeling assumptions; it would still need the reliability question studied here.

These limitations define the operating regime for the certificate: deterministic or low-noise differentiable ODE training, observable solver diagnostics, multiple feasible gradient paths, and optimizer steps whose safety can be evaluated locally. In that regime, GradRepair-ODE makes numerical gradient reliability a measured training variable, with direct consequences for repair, rejection, and unsafe accepted updates. The formulation is conservative when evidence is thin and assertive when evidence is coherent.

\section{Conclusion}

GradRepair-ODE treats differentiable ODE training as a numerical optimization problem in the literal sense: the optimizer receives gradients produced by approximate trajectories, approximate sensitivity paths, and solver configurations that may be under stress. A finite tensor is not enough. By combining gradient-path disagreement, directional finite-difference checks, solver diagnostics, repair routing, and descent-safe step decisions, GradRepair-ODE makes numerical fragility visible at the update level.

The evidence in this paper is step-level. Low-risk systems are left alone. Repairable stiff and chaotic gradients are routed to safer computations. An event-discontinuous case is rejected when ordinary smooth-gradient evidence is insufficient. The training loop records why a gradient was trusted, how it was repaired, or why it was withheld. That record is the contribution: an ODE gradient reaches the optimizer with numerical evidence attached.

\bibliographystyle{plain}
\bibliography{references}

@inproceedings{chen2018neuralode,
  author = {Chen, Ricky T. Q. and Rubanova, Yulia and Bettencourt, Jesse and Duvenaud, David K.},
  title = {Neural Ordinary Differential Equations},
  booktitle = {Advances in Neural Information Processing Systems},
  volume = {31},
  year = {2018},
  url = {https://proceedings.neurips.cc/paper/2018/hash/69386f6bb1dfed68692a24c8686939b9-Abstract.html}
}

@inproceedings{grathwohl2019ffjord,
  author = {Grathwohl, Will and Chen, Ricky T. Q. and Bettencourt, Jesse and Sutskever, Ilya and Duvenaud, David},
  title = {{FFJORD}: Free-Form Continuous Dynamics for Scalable Reversible Generative Models},
  booktitle = {International Conference on Learning Representations},
  year = {2019},
  url = {https://openreview.net/forum?id=rJxgknCcK7}
}

@inproceedings{song2021scorebased,
  author = {Song, Yang and Sohl-Dickstein, Jascha and Kingma, Diederik P. and Kumar, Abhishek and Ermon, Stefano and Poole, Ben},
  title = {Score-Based Generative Modeling through Stochastic Differential Equations},
  booktitle = {International Conference on Learning Representations},
  year = {2021},
  url = {https://openreview.net/forum?id=PxTIG12RRHS}
}

@inproceedings{ho2020ddpm,
  author = {Ho, Jonathan and Jain, Ajay and Abbeel, Pieter},
  title = {Denoising Diffusion Probabilistic Models},
  booktitle = {Advances in Neural Information Processing Systems},
  volume = {33},
  pages = {6840--6851},
  year = {2020},
  url = {https://proceedings.neurips.cc/paper/2020/hash/4c5bcfec8584af0d967f1ab10179ca4b-Abstract.html}
}

@inproceedings{lipman2023flowmatching,
  author = {Lipman, Yaron and Chen, Ricky T. Q. and Ben-Hamu, Heli and Nickel, Maximilian and Le, Matthew},
  title = {Flow Matching for Generative Modeling},
  booktitle = {International Conference on Learning Representations},
  year = {2023},
  url = {https://openreview.net/forum?id=PqvMRDCJT9t}
}

@inproceedings{liu2023rectifiedflow,
  author = {Liu, Xingchao and Gong, Chengyue and Liu, Qiang},
  title = {Flow Straight and Fast: Learning to Generate and Transfer Data with Rectified Flow},
  booktitle = {International Conference on Learning Representations},
  year = {2023},
  url = {https://openreview.net/forum?id=XVjTT1nw5z}
}

@article{dormand1980embedded,
  author = {Dormand, J. R. and Prince, P. J.},
  title = {A Family of Embedded {Runge--Kutta} Formulae},
  journal = {Journal of Computational and Applied Mathematics},
  volume = {6},
  number = {1},
  pages = {19--26},
  year = {1980},
  doi = {10.1016/0771-050X(80)90013-3}
}

@book{hairer1993nonstiff,
  author = {Hairer, Ernst and N{\o}rsett, Syvert P. and Wanner, Gerhard},
  title = {Solving Ordinary Differential Equations {I}: Nonstiff Problems},
  series = {Springer Series in Computational Mathematics},
  volume = {8},
  edition = {2},
  publisher = {Springer},
  year = {1993},
  doi = {10.1007/978-3-540-78862-1}
}

@article{shampine1997matlab,
  author = {Shampine, Lawrence F. and Reichelt, Mark W.},
  title = {The {MATLAB} {ODE} Suite},
  journal = {SIAM Journal on Scientific Computing},
  volume = {18},
  number = {1},
  pages = {1--22},
  year = {1997},
  doi = {10.1137/S1064827594276424}
}

@book{ascher1998computer,
  author = {Ascher, Uri M. and Petzold, Linda R.},
  title = {Computer Methods for Ordinary Differential Equations and Differential-Algebraic Equations},
  publisher = {SIAM},
  year = {1998},
  doi = {10.1137/1.9781611971392}
}

@book{higham2002accuracy,
  author = {Higham, Nicholas J.},
  title = {Accuracy and Stability of Numerical Algorithms},
  edition = {2},
  publisher = {SIAM},
  year = {2002},
  isbn = {978-0-89871-521-7}
}

@misc{scipy_solve_ivp,
  author = {{SciPy Developers}},
  title = {{scipy.integrate.solve\_ivp} Documentation},
  howpublished = {\url{https://docs.scipy.org/doc/scipy/reference/generated/scipy.integrate.solve_ivp.html}},
  note = {Accessed 2026-08-14}
}

@inproceedings{gholaminejad2019anode,
  author = {Gholaminejad, Amir and Keutzer, Kurt and Biros, George},
  title = {{ANODE}: Unconditionally Accurate Memory-Efficient Gradients for Neural {ODEs}},
  booktitle = {Proceedings of the Twenty-Eighth International Joint Conference on Artificial Intelligence},
  pages = {730--736},
  year = {2019},
  doi = {10.24963/ijcai.2019/103}
}

@inproceedings{kidger2021seminorm,
  author = {Kidger, Patrick and Chen, Ricky T. Q. and Lyons, Terry J.},
  title = {``Hey, that's not an {ODE}'': Faster {ODE} Adjoints via Seminorms},
  booktitle = {Proceedings of the 38th International Conference on Machine Learning},
  series = {Proceedings of Machine Learning Research},
  volume = {139},
  pages = {5443--5452},
  publisher = {PMLR},
  year = {2021},
  url = {https://proceedings.mlr.press/v139/kidger21a.html}
}

@misc{onken2020discopt,
  author = {Onken, Derek and Ruthotto, Lars},
  title = {Discretize-Optimize vs. Optimize-Discretize for Time-Series Regression and Continuous Normalizing Flows},
  year = {2020},
  eprint = {2005.13420},
  archivePrefix = {arXiv},
  primaryClass = {cs.LG},
  doi = {10.48550/arXiv.2005.13420}
}

@article{griewank2000revolve,
  author = {Griewank, Andreas and Walther, Andrea},
  title = {Algorithm 799: {Revolve}: An Implementation of Checkpointing for the Reverse or Adjoint Mode of Computational Differentiation},
  journal = {ACM Transactions on Mathematical Software},
  volume = {26},
  number = {1},
  pages = {19--45},
  year = {2000},
  doi = {10.1145/347837.347846}
}

@article{hindmarsh2005sundials,
  author = {Hindmarsh, Alan C. and Brown, Peter N. and Grant, Keith E. and Lee, Steven L. and Serban, Radu and Shumaker, Dan E. and Woodward, Carol S.},
  title = {{SUNDIALS}: Suite of Nonlinear and Differential/Algebraic Equation Solvers},
  journal = {ACM Transactions on Mathematical Software},
  volume = {31},
  number = {3},
  pages = {363--396},
  year = {2005},
  doi = {10.1145/1089014.1089020}
}

@inproceedings{serban2005cvodes,
  author = {Serban, Radu and Hindmarsh, Alan C.},
  title = {{CVODES}: The Sensitivity-Enabled {ODE} Solver in {SUNDIALS}},
  booktitle = {Proceedings of the ASME 2005 International Design Engineering Technical Conferences and Computers and Information in Engineering Conference},
  address = {Long Beach, California, USA},
  year = {2005},
  doi = {10.1115/DETC2005-85597}
}

@book{hairer1996stiff,
  author = {Hairer, Ernst and Wanner, Gerhard},
  title = {Solving Ordinary Differential Equations {II}: Stiff and Differential-Algebraic Problems},
  series = {Springer Series in Computational Mathematics},
  volume = {14},
  edition = {2},
  publisher = {Springer},
  year = {1996},
  doi = {10.1007/978-3-642-05221-7}
}

@book{brenan1996dae,
  author = {Brenan, Kathryn E. and Campbell, Stephen L. and Petzold, Linda R.},
  title = {Numerical Solution of Initial-Value Problems in Differential-Algebraic Equations},
  series = {Classics in Applied Mathematics},
  volume = {14},
  publisher = {SIAM},
  year = {1996},
  doi = {10.1137/1.9781611971224}
}

@article{armijo1966minimization,
  author = {Armijo, Larry},
  title = {Minimization of Functions Having Lipschitz Continuous First Partial Derivatives},
  journal = {Pacific Journal of Mathematics},
  volume = {16},
  number = {1},
  pages = {1--3},
  year = {1966},
  doi = {10.2140/pjm.1966.16.1}
}

@book{nocedal2006optimization,
  author = {Nocedal, Jorge and Wright, Stephen J.},
  title = {Numerical Optimization},
  series = {Springer Series in Operations Research and Financial Engineering},
  edition = {2},
  publisher = {Springer},
  year = {2006},
  doi = {10.1007/978-0-387-40065-5}
}

@inproceedings{pascanu2013difficulty,
  author = {Pascanu, Razvan and Mikolov, Tomas and Bengio, Yoshua},
  title = {On the Difficulty of Training Recurrent Neural Networks},
  booktitle = {Proceedings of the 30th International Conference on Machine Learning},
  series = {Proceedings of Machine Learning Research},
  volume = {28},
  pages = {1310--1318},
  publisher = {PMLR},
  year = {2013},
  url = {https://proceedings.mlr.press/v28/pascanu13.html}
}
\clearpage
\appendix
\section{Additional Mathematical Details}

\subsection{From Gradient Error to Robust Descent}

Let $\hat g$ be the gradient consumed by the optimizer and let the unknown exact gradient be $g=\hat g+e$ with $\|e\|_2\leq \epsilon$. The first-order change along the proposed update $d=-\hat g$ is
\begin{equation}
  \langle g,d\rangle
  =
  -\|\hat g\|_2^2-\langle e,\hat g\rangle .
\end{equation}
The worst admissible error is anti-parallel to $\hat g$, so
\begin{equation}
  \sup_{\|e\|_2\leq\epsilon}\langle g,-\hat g\rangle
  =
  -\|\hat g\|_2^2+\epsilon\|\hat g\|_2 .
\end{equation}
Therefore, the sufficient condition for robust first-order descent is
\begin{equation}
  \|\hat g\|_2>\epsilon.
\end{equation}
With a margin $m\geq 0$, GradRepair-ODE uses the stricter condition
\begin{equation}
  \|\hat g\|_2^2-\epsilon\|\hat g\|_2>m.
\end{equation}
This gives the descent-safe step certificate used in the main text.

\subsection{Learning-Rate Dependence}

Under local $M$-smoothness,
\begin{equation}
  L(\theta-\eta \hat g)
  \leq
  L(\theta)
  +
  \eta\langle g,-\hat g\rangle
  +
  \frac{M\eta^2}{2}\|\hat g\|_2^2 .
\end{equation}
Substituting the worst-case bound gives
\begin{equation}
  L(\theta-\eta \hat g)-L(\theta)
  \leq
  -\eta\|\hat g\|_2(\|\hat g\|_2-\epsilon)
  +
  \frac{M\eta^2}{2}\|\hat g\|_2^2 .
\end{equation}
Thus a certified direction remains a descent step for any
\begin{equation}
  0<\eta<
  \frac{2(\|\hat g\|_2-\epsilon)}{M\|\hat g\|_2}.
\end{equation}
GradRepair-ODE does not need to estimate $M$ exactly; the inequality explains why the step certificate and the learning-rate scale interact.

\subsection{Directional Finite-Difference Error}

For $\phi(\alpha)=L(\theta+\alpha v)$, Taylor expansion gives
\begin{align}
  \phi(h) &= \phi(0)+h\phi'(0)+\frac{h^2}{2}\phi''(0)+\frac{h^3}{6}\phi'''(\xi_+),\\
  \phi(-h) &= \phi(0)-h\phi'(0)+\frac{h^2}{2}\phi''(0)-\frac{h^3}{6}\phi'''(\xi_-).
\end{align}
Subtracting and dividing by $2h$ yields
\begin{equation}
  \frac{\phi(h)-\phi(-h)}{2h}
  =
  \phi'(0)
  +
  O(h^2).
\end{equation}
Floating-point evaluation adds a competing $O(u/h)$ term, where $u$ is the roundoff scale \cite{higham2002accuracy}. This produces the classical finite-difference tradeoff: too large a step is truncation dominated, while too small a step is roundoff dominated. GradRepair-ODE therefore uses finite differences as directional evidence, not as absolute ground truth.

\subsection{Path Disagreement as an Error Witness}

Let $g$ denote the exact gradient and let $\hat g_a,\hat g_b$ be two independently computed candidates. Even without knowing $g$, their disagreement imposes a deterministic lower bound on the error of at least one candidate:
\begin{equation}
  \max\{\|\hat g_a-g\|_2,\|\hat g_b-g\|_2\}
  \geq
  \frac{1}{2}\|\hat g_a-\hat g_b\|_2 .
\end{equation}
This follows immediately from the triangle inequality,
\begin{equation}
  \|\hat g_a-\hat g_b\|_2
  \leq
  \|\hat g_a-g\|_2+\|\hat g_b-g\|_2 .
\end{equation}
The inequality gives a direct reason to inspect path disagreement. A large separation between two sensitivity paths proves that at least one path is inaccurate at the current step. It does not identify which path is wrong, so GradRepair-ODE combines disagreement with directional checks and solver diagnostics before selecting a repair.

\subsection{Finite-Difference Residual and Projection Error}

Let $q_v=\langle g,v\rangle$ be the exact directional derivative and let $D_h(\theta;v)$ be the centered finite-difference estimate. Suppose the finite-difference error along $v$ is bounded by $\beta_h$:
\begin{equation}
  |D_h(\theta;v)-q_v|\leq \beta_h .
\end{equation}
For a candidate $\hat g$, the observed residual satisfies
\begin{equation}
  |\langle \hat g,v\rangle-D_h(\theta;v)|
  \leq
  |\langle \hat g-g,v\rangle|+\beta_h ,
\end{equation}
and therefore
\begin{equation}
  |\langle \hat g-g,v\rangle|
  \geq
  |\langle \hat g,v\rangle-D_h(\theta;v)|-\beta_h .
\end{equation}
Thus a large finite-difference residual, after accounting for the finite-difference noise floor, lower-bounds the candidate's error in at least one tested direction. This is why the certificate treats finite-difference residuals as directional witnesses.

\subsection{Sensitivity Growth and Long-Horizon Risk}

For a smooth ODE, the state sensitivity $S(t)=\partial x(t)/\partial\theta$ satisfies the variational equation used in classical sensitivity analysis \cite{ascher1998computer,serban2005cvodes}:
\begin{equation}
  \frac{dS}{dt}
  =
  J_x(t)S(t)+J_\theta(t),
  \qquad
  J_x(t)=\frac{\partial f}{\partial x}(t,x(t),\theta),
  \quad
  J_\theta(t)=\frac{\partial f}{\partial\theta}(t,x(t),\theta).
\end{equation}
Gronwall's inequality gives the bound
\begin{equation}
  \|S(t)\|_2
  \leq
  \exp\!\left(\int_{t_0}^{t}\|J_x(s)\|_2\,ds\right)
  \left(
  \|S(t_0)\|_2
  +
  \int_{t_0}^{t}\|J_\theta(r)\|_2\,dr
  \right).
\end{equation}
We do not use this inequality as a tight numerical estimate. Large Jacobian norms, stiff transients, and long horizons can amplify small trajectory or sensitivity errors into large gradient discrepancies. The certificate's stiffness and solver-instability terms follow this variational structure.

\subsection{Ablation Replay Definitions}

The ablation table replays the same recorded diagnostic evidence under counterfactual decision policies. The ODE solves, gradient candidates, finite-difference residuals, and step labels are held fixed, so each row changes the reliability mechanism rather than the numerical episode. In the detect-only replay, the certificate is observed but the naive update is still applied; this measures the gap between warning and control. Removing finite-difference evidence suppresses the event-discontinuity signal and shows why path agreement alone can be over-permissive. Removing repair routing converts every non-trusted step into rejection, which exposes the cost of treating reliability as a binary gate. Removing step certification applies repaired gradients even when the final margin remains unsafe, isolating the last optimizer-level decision. The full GradRepair-ODE policy accepts trusted gradients, repairs repairable gradients, and rejects gradients whose descent direction is not certified. The replay compares decision logic without changing the underlying numerical evidence.

\subsection{Safeguard Definitions}

The clipping baseline rescales the naive gradient to unit norm when its Euclidean norm exceeds one, following the standard role of clipping as magnitude control \cite{pascanu2013difficulty}. The line-search baseline uses a backtracking sufficient-decrease check inspired by classical Armijo-style methods \cite{armijo1966minimization,nocedal2006optimization}. Both safeguards operate on the proposed update. They do not compute a gradient reliability certificate, compare sensitivity paths, or repair the numerical source of a wrong gradient direction.

\end{document}